\documentclass[letterpaper]{article} 
\usepackage[]{aaai23}  
\usepackage{times}  
\usepackage{helvet}  
\usepackage{courier}  
\usepackage[hyphens]{url}  

\usepackage{graphicx} 
\usepackage{natbib}  
\usepackage{caption} 
\usepackage{algorithm}
\usepackage{algorithmic}
\usepackage{amsmath}
\usepackage{amssymb}
\usepackage{mathtools}
\usepackage{amsthm}
\usepackage{bbm}
\usepackage{diagbox}
\usepackage{makecell}
\usepackage{multirow}
\usepackage{multicol}
\usepackage[table,xcdraw]{xcolor}
\theoremstyle{plain}
\newtheorem{theorem}{Theorem}

\theoremstyle{definition}
\newtheorem{definition}{Definition}

\theoremstyle{remark}

\usepackage{newfloat}
\usepackage{listings}
\DeclareCaptionStyle{ruled}{labelfont=normalfont,labelsep=colon,strut=off} 
\floatstyle{ruled}
\newfloat{listing}{tb}{lst}{}
\floatname{listing}{Listing}
\nocopyright 

\title{Accurate Fairness: Improving Individual Fairness without Trading Accuracy}
\author{
    Xuran Li, 
   Peng Wu
   \thanks{Corresponding author: Peng Wu}, 
   Jing Su
}
\affiliations{
    State Key Laboratory of Computer Science, Institute of Software, Chinese Academy of Sciences, Beijing, China \\
    University of Chinese Academy of Sciences, Beijing, China \\
    {lixr, wp, sujing}@ios.ac.cn
}

\begin{document}

\maketitle

\begin{abstract}
Accuracy and individual fairness are both crucial for trustworthy machine learning, but these two aspects are often incompatible with each other so that enhancing one aspect may sacrifice the other inevitably with side effects of true bias or false fairness.
We propose in this paper a new fairness criterion, \emph{accurate fairness}, to align individual fairness with accuracy. Informally, it requires the treatments of an individual and the individual's similar counterparts conform to a uniform target, i.e., the ground truth of the individual. We prove that accurate fairness also implies typical group fairness criteria over a union of similar sub-populations.
We then present a \emph{Siamese fairness} in-processing approach to minimize the accuracy and fairness losses of a machine learning model under the accurate fairness constraints. To the best of our knowledge, this is the first time that a Siamese approach is adapted for bias mitigation. We also propose fairness confusion matrix-based metrics, \emph{fair-precision}, \emph{fair-recall} and \emph{fair-F1 score}, to quantify a trade-off between accuracy and individual fairness. Comparative case studies with popular fairness datasets show that our Siamese fairness approach can achieve on average 
1.02\%-8.78\% higher individual fairness (in terms of fairness through awareness) and 
8.38\%-13.69\% higher accuracy, as well as 
10.09\%-20.57\% higher true fair rate and 
5.43\%-10.01\% higher fair-F1 score, than the state-of-the-art bias mitigation techniques. This demonstrates that our Siamese fairness approach can indeed improve individual fairness without trading accuracy. Finally, the accurate fairness criterion and Siamese fairness approach are applied to mitigate the possible service discrimination with a real Ctrip dataset, by on average fairly serving 
112.33\% more customers (specifically, 
81.29\% more customers in an accurately fair way) than baseline models.
\end{abstract}

\section{Introduction}
Machine learning aided intelligent systems have exhibited competitive performances in decision-making tasks such as loan granting \cite{equal-odds-1}, criminal justice risk assessment \cite{machine-learning-survey-2}, and online recommendations \cite{recommend-1}. 
However, the widespread deployments of such machine-learning systems have also spawned social and political concerns, 
particularly on the fairness of the decisions or predictions made by these systems.

Accuracy and fairness are both crucial for trustworthy machine learning \cite{robustness1,robustness2, OOD, Adversarial, machine-learning-survey-1}, but these two aspects may be incompatible fundamentally from their own unilateral perspectives, that is, enhancing one aspect may sacrifice the other inevitably with unacceptable consequences \cite{trade-off1,trade-off2, AAAI_2022}. 
For instance, more accurate predictions on loan applicants' incomes can benefit banks with less lending risks, but the underlying ground truth distribution may tend to prefer applicants with the majority or privileged backgrounds, due to historical practices.
Thus, accurate predictions would reflect, even exaggerate such discrimination against minority or unprivileged applicants. 
In contrast, enhancing just fairness, e.g., by blindly enforcing all the applicants to have the same access to loans, would result in trivially fair but unsound predictions for actually non-qualified applicants. 
Therefore, accurate but biased, and fair but faulty predictions do not yield a mutually beneficial trade-off between accuracy and fairness. Such incompatibility has recently been shown in \cite{AAAI_2022} specifically between non-trivial accuracy and equal opportunity, a group fairness criterion.

In this paper, we propose a new fairness criterion, \emph{accurate fairness}, to align individual fairness \cite{individually-fairness1,causal-discrimination2017} with accuracy by uniformly bounding both the accuracy difference and the fairness difference for \emph{similar} sub-populations. Any two individuals are \emph{similar} if both differ only on their sensitive attributes, e.g., genders, races, and ages. Then, an individual is treated by a machine learning model in an \emph{accurately fair} way, if its prediction results for both the individual and the individual's similar counterparts conform to the ground truth of the individual; otherwise, the prediction result for this individual is either faulty or biased. Thus, under the notion of accurate fairness, an individual and the individual's similar counterparts shall be treated similarly in conformance with a uniform target (i.e., the ground truth of the individual), without acknowledging their differences in the sensitive attributes.

As an individual-level fairness criterion, accurate fairness refines the general definition of individual fairness \cite{individually-fairness1} by explicitly focusing on similar sub-populations, where the individuals are exactly the same on their non-sensitive attributes, instead of on any individuals that are close to each other. Thus, fair but faulty predictions can be potentially reduced because a machine learning model does not have to learn ``fair" predictions for close individuals, who though differ on some of their non-sensitive attributes, without regard to their different ground truths. Accurate fairness further captures the intuition that fairness criteria shall be truthfully built upon accurate predictions. Consequently, as a by-product, we show that accurate fairness implies group fairness, specifically statistical parity \cite{Statistical-parity-1, Statistical-parity-2} and confusion matrix-based fairness \cite{machine-learning-survey-3}) over a union of similar sub-populations.  

We then present and implement a Siamese fairness approach to mitigate individual bias without trading accuracy. It simultaneously receives multiple similar individuals as training inputs, and aims to minimize the accuracy and fairness losses of a machine learning model (e.g., a neural network model, a logistic regression model, or a support vector machine) under the accurate fairness constraints.   
We further propose a \emph{fairness confusion matrix} to evaluate how well a machine learning model can balance accuracy with individual fairness, yielding \emph{fair-precision}, \emph{fair-recall}, and \emph{fair-F1 score} metrics. Fair-precision is the proportion of individually fair predictions in accurate predictions, while fair-recall is the proportion of accurate predictions in individually fair predictions. Fair-F1 score is the harmonic mean of fair-precision and fair-recall. 

Empirical studies with popular fairness datasets Adult \cite{adult}, German Credit \cite{credit} and ProPublica Recidivism \cite{COMPAS}, show that the accurate fairness criterion contributes well to delivering a truthfully fair solution for decision-making, and balances accuracy and individual fairness in a win-win manner. Compared with the state-of-the-art bias mitigation techniques, our Siamese fairness approach can on average promote the individual fairness (fairness through awareness) of a machine learning model 1.02\%-8.78\% higher, and the model accuracy 8.38\%-13.69\% higher, with 10.09\%-20.57\% higher true fair rate and 5.43\%-10.01\% higher F-F1 score.

Finally, we apply the accurate fairness criterion to evaluate a service discrimination problem with a real dataset \cite{ctrip} from Ctrip, one of the largest online travel service providers in the world. This problem concerns whether customers who pay the same prices for the same rooms are recommended the same room services, irrespective of their consumption habits. Two neural network models are trained as baseline models, which do suffer service discrimination against customers with different consumption habits. Our Siamese fairness approach can mitigate such discrimination to a great extent, by on average fairly serving 93.00\% customers (112.33\% more than the baseline models). More importantly, 81.29\% more customers are served in an accurately fair way.

The main contributions of this paper are as follows.
\begin{itemize}
\item We propose an individual level fairness criterion, \emph{accurate fairness}, such that any individual and the individual's similar counterparts shall all be treated similarly up to the ground truth of the individual. This makes it a new individual fairness criterion that is accuracy-enhanced and can imply certain group-level fairness criteria in the context of sub-populations.

\item We present and implement a Siamese fairness approach to optimize the accurate fairness of a machine learning model, by taking similar individuals as parallel training inputs. To the best of our knowledge, this is the first time that a Siamese approach is adapted for individual bias mitigation. 

\item The accurate fairness criterion and the Siamese fairness approach are applied with popular fairness datasets and a real Ctrip dataset, under the evaluation of what we propose as fairness confusion matrix-based metrics: fair-precision, fair-recall, and fair-F1 score. The case studies reveal the defects of true bias and false fairness in the learned classifiers. Our approach can indeed mitigate these defects and improve individual fairness without trading accuracy.
\end{itemize}

The rest of this paper is organized as follows. 
We briefly discuss the related work in Section 2, followed by the formal definition and discussion of the accurate fairness criterion in Section 3. We present the Siamese fairness approach in Section 4. Its implementation and evaluation results are reported and analyzed in Section 5. The paper is concluded in Section 6 with some future work.

\section{Related Work}

\subsection{Fairness Criteria}
Fairness criteria presented in the literature are usually partitioned into two categories: group fairness and individual fairness. Please refer to \cite{causal-discrimination2017,individually-fairness1,counterfactual-fairness-1,machine-learning-survey-3,machine-learning-survey-1,machine-learning-survey-2} for a comprehensive survey about machine learning fairness notions.

Group fairness criteria concern equal treatments for sub-groups with the same sensitive attribute values, and hence are usually defined statistically in terms of conditional independence.
Statistical parity \cite{Statistical-parity-1, Statistical-parity-2, Statistical-parity-3} requires predictions independent of sensitive attributes so that all the sub-groups have the same positive prediction rates. Confusion matrix-based fairness criteria, e.g., equality odds \cite{equal-odds-1} and accuracy equality \cite{machine-learning-survey-2}, require predictions independent of sensitive attributes under the given ground truths. However, group fairness criteria may be satisfied with carefully selected individuals, who are unfavorably discriminated against in contrast to their similar counterparts \cite{machine-learning-survey-1}. Thus, individual fairness for these individuals is unnecessarily neglected. 

Individual fairness criteria can be defined qualitatively or quantitatively by interpreting the notions of \emph{similar individuals} and \emph{similar treatments} \cite{individual_distance}, in order to assess whether similar individuals are treated similarly. Causal discrimination \cite{causal-discrimination2017, if-test} is such a qualitative definition, where similar individuals are those who differ only on sensitive attributes, and only equal predictions are accounted as similar treatments.In a quantitative or algorithmic definition, task-specific distance metrics are adapted to characterize the similarities between individuals and between prediction outcomes. Fairness through awareness \cite{individually-fairness1} requires that the similarity distance between individuals lays an upper bound on the similarity distance between the corresponding prediction outcomes by the Lipschitz condition. 

Individual fairness criteria concern directly the predictions themselves, which nonetheless are possibly (partly) faulty. Accurate fairness presented in this paper refactors the individual fairness criteria from a viewpoint of accuracy to clarify and quantify such incompatibility and also implies certain group fairness criteria over sub-populations. 

\subsection{Bias Mitigation}
As summarized in \cite{machine-learning-survey-3, AIF360}, the bias of a machine learning model can be mitigated through pre-processing the training data, in-processing the model itself, or post-processing the predictions \cite{post-processing}. 

Pre-processing methods aim to learn non-discriminative data representations \cite{data_augmentation}. A fair representation learning (LFR) approach \cite{lfr} obfuscates any information about sensitive attributes in the learned data representation. iFair \cite{iFair} minimizes the data loss to reconcile individual fairness with application utility.

In-processing methods train a machine learning model with fairness as an additional optimization goal. SenSR \cite{SensitiveSubspaceRobustness} improves the sensitive subspace robustness against certain sensitive perturbations through a distributionally robust optimization approach. SenSeI \cite{SenSeI} enforces the treatment invariance on certain sensitive sets by minimizing a transport-based regularizer through a stochastic approximation algorithm. 

These methods separate model accuracy from mitigating individual bias and hence may unilaterally improve individual fairness with accuracy decreasing. 
Our Siamese fairness approach minimizes the model accuracy and fairness losses uniformly subject to the new accurate fairness criterion, so mitigating individual bias does not necessarily trade accuracy.

\section{Accurate Fairness}
We present in this section the notion of accurate fairness and discuss its connections with other individual fairness and group fairness criteria.

Assume a finite and labeled dataset $V\subseteq X\times A$ with the domains of the sensitive attributes, the non-sensitive attributes, and the ground truth labels denoted $A, X, Y$, respectively. Each input $(x,a)\in V$ is associated with a ground truth label $y\in Y$. Let $I(x,a)\subseteq X\times A$ be the similar sub-population of $(x,a)$, which is the set of the individuals sharing the same non-sensitive attributes values with $(x,a)$, i.e.,  
\begin{equation*}
 I(x,a)=\{(x,a')~|~a'\in A\}
\end{equation*}
Obviously, $(x,a)\in I(x,a)$. Let $card(S)$ be the cardinal number of set $S$.

Let $f:X\times A\rightarrow {Y}$ denote a classifier learned from a training dataset $V$, and $\hat{y}=f(x,a)$ the prediction result of classifier $f$ for input $(x,a)$. Then, the accurate fairness criterion can be defined as follow. 

\begin{definition}[Accurate Fairness]\label{af-definition}
A classifier $f:X\times A\rightarrow {Y}$ is \emph{accurately fair} to an input $(x,a)\!\in\!V$, if for any individual $(x,a') \in I(x,a)$, the distance $D(y,f(x,a'))$ between the ground truth $y$ of input $(x,a)$ and the prediction result $f(x,a')$ is at most $K\geq 0$ times of the distance $d((x,a),(x,a')$ between $(x,a)$ and $(x,a')$, i.e.,
\begin{gather}
 D(y,f(x,a')) \leq Kd((x,a),(x,a'))
\end{gather}
\label{eqn:afc}
where $D(\cdot,\cdot)$ and $d(\cdot,\cdot)$ are distance metrics.
\end{definition}

Herein, the accurate fairness constraint \eqref{eqn:afc} captures uniformly the accuracy and individual fairness requirements with respect to input $(x,a)\in V$:
\begin{itemize}
\item (Accuracy) Since $(x,a)\in I(x,a)$, constraint \eqref{eqn:afc} reduces to $D(y,f(x,a))=0$ (i.e., $y=f(x,a)$ due to the identity of indiscernibles of a distance metric) for input $(x,a)$ itself;
\item (Individual Fairness) For any similar individual $(x,a')\in I(x,a)$ with $a'\neq a$, constraint \eqref{eqn:afc} reduces the Lipschitz condition \cite{individually-fairness1} $D(f(x,a),f(x,a')) \leq Kd((x,a),(x,a'))$ for similar individual $(x,a')$ within sub-population $I(x,a)$, as shown by the following theorem.
\end{itemize}

\begin{theorem}
\label{IF-theorem}
If a classifier $f:X\times A\rightarrow{Y}$ is accurately fair to an input $(x,a) \in V$, then
\begin{center}
$D(f(x,a),f(x,a')) \leq Kd((x,a),(x,a'))$
\end{center}
for any similar individual $(x,a')\in I(x,a)$ with $a'\neq a$.
\end{theorem}
\begin{proof}
Due to the triangle inequality and symmetry of a distance metric,
\begin{align*}
 D(f(x,a),f(x,a')) \leq K(D(y,f(x,a))+D(y,f(x,a')))
\end{align*}
where $a'\neq a$ and $y$ is the ground truth of input $(x,a)\in V$. 

Then, By Definition \ref{af-definition}, if classifier $f$ is accurately fair to input $(x,a)$, for any similar individual $(x,a')\in I(x,a)$ with $a'\neq a$,
\begin{align*}
 D(y,f(x,a)) &\leq Kd((x,a),(x,a)) =0\\
 D(y,f(x,a')) &\leq Kd((x,a),(x,a'))
\end{align*}
Thus, $D(f(x,a),f(x,a')) \leq Kd((x,a),(x,a'))$
\end{proof}

It can be seen from Theorem \ref{IF-theorem} that accurate fairness refactors the general definition of individual fairness \cite{individually-fairness1} over similar sub-populations on the basis of accuracy.

\newcommand{\V}[0]{W}
\newcommand{\IV}[0]{I(\V)}
\newcommand{\IaV}[1]{I_{#1}(\V)}
Accurate fairness also collectively endorses group fairness criteria over the union of similar sub-populations. Consider the following definition of accurate parity, which is a qualitative version of accurate fairness.
\begin{definition}[Accurate Parity]\label{ap-definition}
A classifier $f:X\times A\rightarrow{Y}$ is \emph{accurately equal} to an input $(x,a)\!\in\!V$, if for any individual $(x,a') \in I(x,a)$, 
\begin{gather}
y=f(x,a)=f(x,a')
\end{gather}
\label{eqn:apc}
where $y$ is the ground truth of input$(x,a)$.
\end{definition}
Obviously accurate parity entails accurate fairness because with the accurate parity constraint \eqref{eqn:apc}, $D(y,f(x,a'))=0$ for any individual $(x,a')\in I(x,a)$. 

Let $\mathbf{X}$, $\mathbf{A}$, $\mathbf{Y}$, and $\mathbf{\hat{Y}}$ denote the random variables representing the non-protected attributes, the protected attributes, the ground truths, and the prediction results. The following theorem shows that accurate parity implies statistical parity and confusion matrix-based fairness over $\IV={\cup}_{(x,a)\in\V} I(x,a)$ for certain $\V\subseteq V$. As far as the group fairness criteria are concerned, assume each individual $(x, a') \in\IV$ is associated with the same ground truth as some $(x, a) \in\V$ is.
\begin{theorem}
\label{thm:both}
If a classifier $f : X \times A \to {Y}$ is accurately equal to each input in $W \subseteq V$, then $f$ satisfies statistical parity and confusion matrix-based fairness over $\IV$.
\end{theorem}
\begin{proof}
The accurate parity constraint \eqref{eqn:apc} implies that
\begin{center}
$\mathbf{P}(\mathbf{\hat{Y}}=\mathbf{Y}|\mathbf{A}=a)=\mathbf{P}(\mathbf{\hat{Y}}=\mathbf{Y}|\mathbf{A}=a')$
\end{center}
over $\IV$ for any $a\neq a'$, and hence
\begin{center}
$\mathbf{P}(\mathbf{\hat{Y}}\neq\mathbf{Y}|\mathbf{A}=a)=\mathbf{P}(\mathbf{\hat{Y}}\neq\mathbf{Y}|\mathbf{A}=a')$
\end{center}
over $\IV$. Thus, the accuracy (or inaccuracy) for the similar individuals are independent on the sensitive attributes. Therefore, $f$ satisfies the confusion matrix-based fairness over $\IV$. Furthermore, $f$ also satisfies statistical parity over $\IV$, because $\mathbf{P}(\mathbf{\hat{Y}}=\hat{y}|\mathbf{A} = a) = \mathbf{P}(\mathbf{\hat{Y}}=\hat{y},\mathbf{\hat{Y}} = \mathbf{Y} | \mathbf{A} = a)+\mathbf{P}(\mathbf{\hat{Y}}=\hat{y},\mathbf{\hat{Y}}\neq\mathbf{Y}|\mathbf{A}=a)$ for any $\hat{y}\in{Y}$.
\end{proof}

Note that Proposition 3 in \cite{barocas-hardt-narayanan} shows that statistical parity (independence) and confusion matrix-based fairness (separation) cannot both hold unless $\mathbf{A}\bot\mathbf{Y}$ or $\mathbf{\hat{Y}}\bot\mathbf{Y}$, while the former is admitted on $\IV$ under the accurate parity constraints. 

Generally speaking, accurate fairness lays an upper bound on the treatment differences between groups with different sensitive attribute values. Let $\IaV{a}$ be the set of individuals in $\IV$ with the same sensitive attribute values $a\in A$, i.e., $\IaV{a}=\{(x,a)|(x,a) \in \IV\}$.  

\begin{theorem}
\label{loss-theorem}
If a classifier $f:X\times A\rightarrow{Y}$ is accurately fair to each input in $\V\subseteq V$, for any $(x,a^*)\in\V$, $a\in A$ and $a\neq a^*$, then over $\IaV{a}$, 
\begin{eqnarray*}
\mathbf{E}[D(\mathbf{Y},f(\mathbf{X},a))] & \leq & K\mathbf{E}[d((\mathbf{X},a^*),(\mathbf{X},a))]\\
\mathbf{E}[D(f(\mathbf{X},a^*),f(\mathbf{X},a))] & \leq & K\mathbf{E}[d((\mathbf{X},a^*),(\mathbf{X},a))]
\end{eqnarray*}
\end{theorem}
\begin{proof}
Recall that over $\IaV{a}$, 
\begin{equation*}
 \mathbf{E}[D(\mathbf{Y},f(\mathbf{X},a))] = \sum_{(x,a) \in \IaV{a}}\mathbf{P}(y,x,a)D(y,f(x,a))
\end{equation*}
where $y$ is the ground truth of $(x,a^*)$ and $\mathbf{P}(y,x,a)$ is the joint probability of $\mathbf{Y}=y$ and $\mathbf{(X,A)}=(x,a)$.
\begin{equation*}
\mathbf{E}[d((\mathbf{X},a^*),(\mathbf{X},a))] = \sum_{(x,a) \in \IaV{a}}\mathbf{P}(x,a)d((x,a^*),(x,a))
\end{equation*}
where $\mathbf{P}(x,a)$ is the probability of $\mathbf{(X,A)}=(x,a)$.
\begin{equation*}
 \mathbf{E}[D(f(\mathbf{X},a^*),f(\mathbf{X},a))] = \sum_{(x,a) \in \IaV{a}}\mathbf{P}(x,a)D(f(x,a^*),f(x,a))
\end{equation*}

Then, since $f$ is accurately fair to each input in $\V$, the proof is concluded by Definition \ref{af-definition}.
\end{proof}

Theorems \ref{loss-theorem}
suggest that under the accurate fairness criterion, the treatment differences between individuals and their similar counterparts are also bounded by distances between these individuals themselves, hence leading to quantitatively fair treatments over groups with difference sensitive attributes values. 

\section{Siamese In-Processing for Accurate Fairness}
We present in this section the Siamese fairness in-processing approach to achieve accurate fairness. It intends to train a machine learning model for the following optimization problem, which minimizes the cumulative loss for the union $I(V)$ of similar populations, subject to the accurate fairness constraints.
\begin{gather}
\min_{f} \sum_{(x_i,a_i)\in V}\sum_{(x_i,a_{ij})\in I(x_i,a_i)} L(y_{i},f(x_i,a_{ij})) \label{primary_optimization}\\
\text{s.t. } D(y_{i}, f(x_i,a_{ij})) \leq Kd((x_i,a_i),(x_i,a_{ij})) \notag \\
\quad\text{for any }(x_i,a_i)\in V,(x_i,a_{ij})\in I(x_i,a_i) \notag
\end{gather}
where $y_{i}$ is the ground truth of $(x_i,a_i)\in V$, $1\leq i\leq card(V)$, $1\leq j\leq card(I(x_i,a_i))$ and $L(\cdot, \cdot)$ is a loss function for training the machine learning model. 

By appealing to Karush–Kuhn–Tucker conditions \cite{convex}, it is equivalent to solve the following max-min optimization problem with the Lagrange multipliers $\lambda_{ij}\geq 0$ for each $(x_i,a_i)\in V,(x_i,a_{ij})\in I(x_i,a_i)$, assuming that the loss function $L(\cdot, \cdot)$ and the distance metrics $D(\cdot, \cdot)$ and $d(\cdot, \cdot)$ are all convex.
\begin{align}\label{duality_optimization}
\max_{\lambda}\min_{f} &\sum_{(x_i,a_i)\in V}\sum_{(x_i,a_{ij})\in I(x_i,a_i)} \Big(L(y_i,f(x_i,a_{ij}))+
\\&\lambda_{ij} \big(D(y_i,f(x_i,a_{ij}))- Kd((x_i,a_i),(x_i,a_{ij}))\big)\Big) \notag
\end{align} 

It can be seen that the objective function in \eqref{duality_optimization} renders a possibility of stochastic estimation with observations on the union $I(V)$ of the similar sub-populations, instead of just the training dataset $V\subseteq I(V)$. A Siamese network can accept multiple inputs in parallel to train multiple models with shared parameters \cite{Siamese-network}. Thus, it provides a training mechanism to treat individuals in a similar sub-population in a uniform manner.

Therefore, we propose to adapt a Siamese network for accurate fairness in-processing. The architecture of our approach is shown in \figurename~\ref{pic:architecture_of_SF}. It first generates the similar sub-population $I(x_i,a_i)$ for each training input $(x_i,a_i)\in V$ with $1\leq i\leq card(V)$ through fair augmentation. Then, it trains $m=card(I(x_i,a_i))$ copies of a machine learning model with shared parameters $\theta$ for the sake of minimizing the accurate fairness loss $L_{AF}(x_i,a_i,\lambda_i)$ over the similar sub-population $I(x_i,a_i)$:
\begin{align*}
\label{AF_loss_function}
L_{AF}(x_i,a_i,\lambda_i)=\sum_{(x_i,a_{ij}) \in I(x_i,a_i)} \Big(L(y_i,f(x_i,a_{ij}))+
\\\lambda_{ij} (D(y_i,f(x_i,a_{ij}))- Kd((x_i,a_i),(x_i,a_{ij})))\Big)
\end{align*}
where $\lambda_i=(\lambda_{i1},\cdots,\lambda_{ij},\cdots,\lambda_{im})$.

\begin{figure}[!htb]
 \centering
 \centerline{\includegraphics[scale=0.20, trim=0 30 0 30]{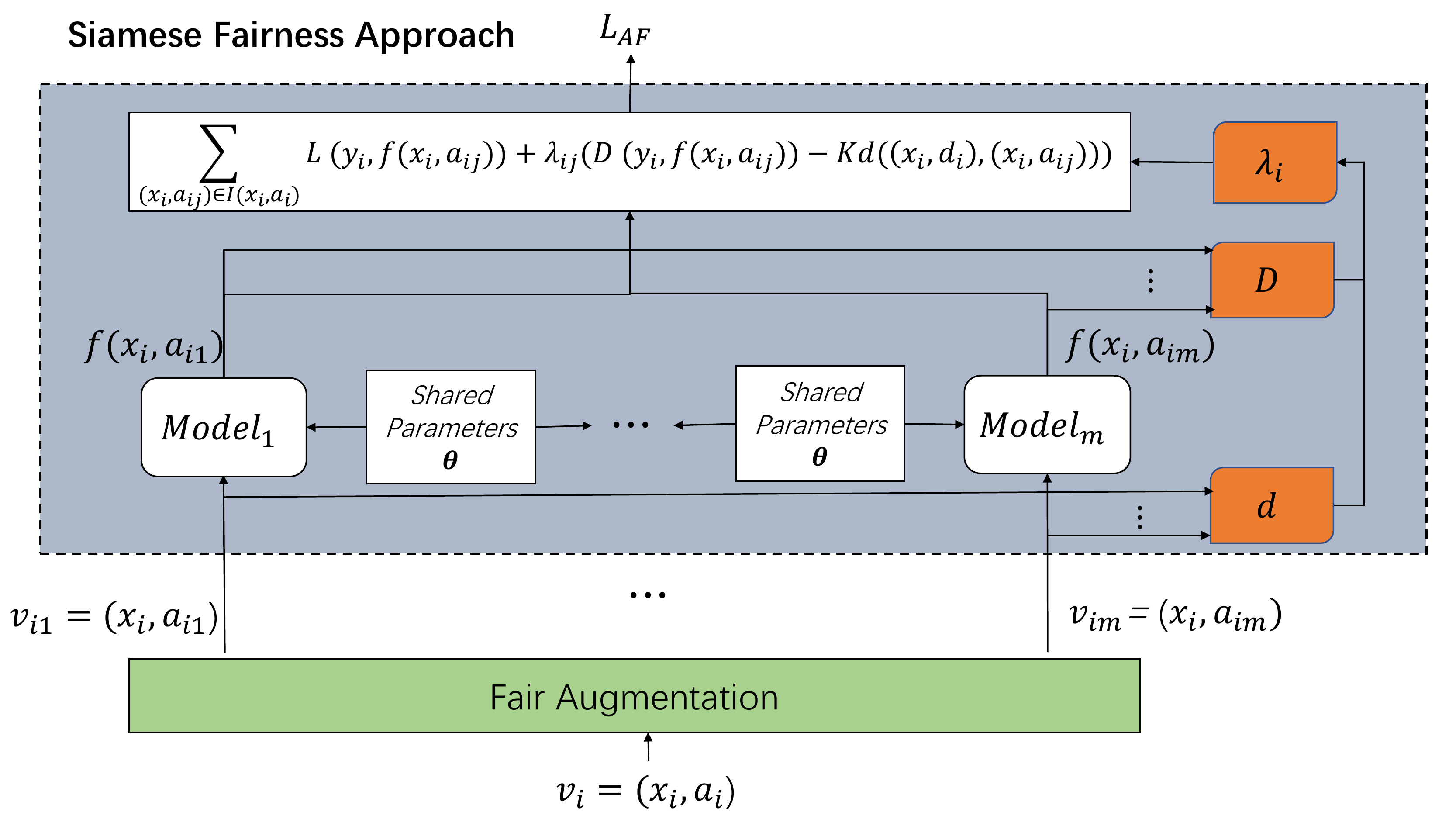}}
 \caption{Siamese fairness approach}
 \label{pic:architecture_of_SF}
 \vskip -0.1in
\end{figure}

Algorithm \ref{alg:SF_algorithm} shows the workflow of our Siamese fairness approach in detail. At Lines 1-6, the training dataset $V$ is augmented with the similar counterparts of each input $v_i=(x_i,a_i)\in V$, resulting in $I(V)$ for the subsequent Siamese training. At Lines 10-12, each $Model_j$ in \figurename~\ref{pic:architecture_of_SF} run a copy of classifier $f_\theta$ with the shared parameters $\theta$, 
accepting the $j$-th similar individual $(x_i,a_{ij})\in I(x_i,a_i)$ and producing $f_\theta(x_i,a_{ij})$ for $1\leq j\leq card(I(x_i,a_i))$. At Lines 13-16,  the Lagrange multipliers $\lambda_i$ and the shared parameters $\theta$ for $v_i$ are obtained by applying an error Back-Propagation (BP) algorithm \cite{BP, BP-1, BP-2} to optimize $\sum_{(x_i,a_i)\in V}L_{AF}(x_i,a_i,\lambda_i)$ (i.e., the objective function in \eqref{duality_optimization}).

\begin{algorithm}[!hbt]
\caption{Siamese Fairness (SF)}
\label{alg:SF_algorithm}
\textbf{Input}: dataset $V$, classifier $f_\theta$, learning rate $\eta$\\
\textbf{Output}: classifier $f_\theta$
\begin{algorithmic}[1] 
\FOR{each \(v_i=(x_i,a_i) \in V\)} 
\STATE $I(x_i,a_i)\gets\{v_i\}$;
\FOR{each $a\in A$ and $a\neq a_i$}
\STATE $I(x_i,a_i)\gets I(x_i,a_i)\cup\{(x_i,a)\}$
\ENDFOR
\ENDFOR
\STATE Initialize $\lambda_1,\cdots,\lambda_{card(V)}$ and  \(\theta\);
\REPEAT
\FOR{each \(v_i=(x_i,a_i) \in V\)}
\FOR{each \((x_i,a_{ij})\in I(x_i,a_i)\)}
\STATE Compute $f_\theta(x_i,a_{ij})$;
\ENDFOR
\FOR{each \(\lambda_{ij}\in\lambda_i, w\in\theta\)}
\STATE \(\lambda_{ij}\leftarrow \max (0,\lambda_{ij}+\eta \frac{\partial L_{AF}(x_i,a_i,\lambda_i)}{\partial \lambda_{ij}})\);
\STATE \(w\leftarrow w-\eta \frac{\partial L_{AF}(x_i,a_i,\lambda_i)}{\partial w}\)
\ENDFOR
\ENDFOR
\UNTIL{$\theta$ converges or the maximal number of iterations is reached}
\STATE \textbf{return} \(f_\theta\)
\end{algorithmic}
\end{algorithm}

The Siamese in-processing architecture in \figurename~\ref{pic:architecture_of_SF} allows treating individuals in one similar sub-population simultaneously and uniformly as a whole during each iteration of back-propagation in Algorithm \ref{alg:SF_algorithm}, while classical training algorithms usually handle inputs one by one, unable to accommodate the accurate fairness criterion for bias mitigation.

\section{Implementation and Evaluation}
We implement the Siamese fairness approach (Algorithm \ref{alg:SF_algorithm}) in Python 3.8 with TensorFlow 2.4.1. Our implementation is evaluated on a Ubuntu 18.04.3 system with Intel Xeon Gold 6154 @3.00GHz CPUs, GeForce RTX 2080 TI GPUs, and 512G memory, in comparison with the state-of-the-art individual fairness bias mitigation techniques with regard to binary or multi-valued sensitive attributes, or the combinations thereof. The source code and the experimental datasets and models are available at \url{https://github.com/Xuran-LI/AccurateFairnessCriterion}. 

\subsection{Datasets and Models}
The three popular fairness datasets Adult, German Credit and COMPAS, and a real dataset from Ctrip are used for the evaluation. The instances with unknown or empty values have been removed from the datasets before training. \tablename~\ref{tab-datasets} reports the size and the sensitive attributes of each dataset, and the models trained with these datasets. ``$A$($m$-valued)" means that attribute $A$ has $m$ values. An FCNN($l$) model is a fully connected neural network (FCNN) classifier with $l$ layers; while an LR or SVM model is a logistic regression (LR) or Support Vector Machine (SVM) classifier, respectively. These classifiers are referred to as the baseline (BL) models in the evaluation. 
\begin{table}[!htb]
 \centering
 \resizebox{\columnwidth}{!}{
\begin{tabular}{cccc}
\hline
Dataset & Size & Models & Sensitive Attributes \\ \hline
\makecell*[c]{Adult (Census Income) } & 45222 & \makecell*[c]{FCNN(3)\\LR\\SVM}&\makecell*[c]{gender (binary) \\ age (71-valued) \\ race (5-valued) } \\ \hline 
\makecell*[c]{German Credit } & 1000 & \makecell*[c]{FCNN(3)\\LR\\SVM} & \makecell*[c]{gender (binary) \\ age (51-valued)} \\ \hline
\makecell*[c]{ProPublica Recidivism \\(COMPAS) } & 6172 & \makecell*[c]{FCNN(3)\\LR\\SVM} & \makecell*[c]{gender (binary) \\ age (71-valued)\\ race (6-valued)} \\ \hline
\makecell*[c]{Ctrip } & 68191 & \makecell*[c]{FCNN(3)\\FCNN(5)} & \makecell{ 6 customer \\ consumption habits} \\ \hline
\end{tabular}
}
\caption{Datasets and Models}
\label{tab-datasets}
\end{table}

\subsection{Evaluation Metrics}
In addition to the accuracy metric (ACC), the group fairness metrics (including statistical parity difference (SPD), equal odds difference (EOD), and average odds difference (AVOD)), and the individual fairness metrics (including fairness through awareness (FTA), consistency (CON)), we propose a fairness confusion matrix (as shown in \tablename~\ref{tab:FCM}), and the following fairness confusion matrix based metrics to evaluate the bias mitigation performance of a machine learning model in balancing its accuracy with individual fairness. 
\begin{table}[!hbt]
\centering
 \setlength{\tabcolsep}{2.8mm}{
    \begin{tabular}{c|cc} \hline
    \diagbox{Accuracy}{Fairness} & Fair & Biased \\ \hline
    True & True Fair & True Biased \\ 
    False & False Fair & False Biased \\ \hline
    \end{tabular}
 }
  \caption{Fairness Confusion Matrix}
  \label{tab:FCM}
\end{table}

\begin{definition}[Fairness Confusion Matrix Based Metrics]\label{FCM_definition}
For classifier $f:X\times A\rightarrow{Y}$ and input $(x,a)\in V$,
\begin{itemize}
\item the prediction $f(x,a)$ is \emph{true fair} if the prediction $f(x,a')$ for any $(x,a')\in I(x,a)$ conforms to $y$, the ground truth of $(x,a)$;
\item $f(x,a)$ is \emph{true biased} if it conforms to $y$, but the prediction $f(x,a')$ for some $(x,a')\in I(x,a)$ with $a'\neq a$ does not;
\item $f(x,a)$ is \emph{false fair} if it does not conform to $y$, but is consistent to the prediction $f(x,a')$ for any $(x,a')\in I(x,a)$ with $a'\neq a$; 
\item $f(x,a)$ is \emph{false biased} if neither the predictions $f(x,a')$ for all $(x,a')\in I(x,a)$ are consistent to each other, nor $f(x,a)$ conforms to $y$. 
\end{itemize}

Let \emph{True Fair Rate} (TFR), \emph{True Biased Rate} (TBR), \emph{False Fair Rate} (FFR), \emph{False Biased Rate} (FBR) be the proportion of the true fair, true biased, false fair, false biased predictions in all the predictions on $V$, respectively. Then,
\emph{Fair-Precision} (F-P), \emph{Fair-Recall} (F-R) and \emph{Fair-F1 Score} (F-F1) can be defined as follows:
\begin{align*}
F\mbox{-}P = \dfrac { TFR }{ TFR\!+\!TBR } & \quad & F\mbox{-}R = \dfrac { TFR }{ TFR\!+\!FFR } \\
\multicolumn{3}{c}{$F\mbox{-}F1 = \dfrac { 2\times {F\mbox{-}P}\times F\mbox{-}R }{ F\mbox{-}P+F\mbox{-}R}$}
\end{align*}
\end{definition}

Informally, the fairness confusion matrix summarizes the orthogonal synergy between individual fairness and accuracy. $F\mbox{-}P$ measures the individually fair proportion in the accurate predictions, while $F\mbox{-}R$ measures the accurate proportion in the individually fair predictions. $F\mbox{-}F1$ combines fair-precision and fair-recall to measure the compatibility between accuracy and individual fairness. 

\subsection{Mitigating Individual Bias}
\label{sec-5-2}
\tablename ~\ref{tab:average_result} reports the average statistics of ten runs for each bias mitigation approach compared over the three fairness datasets. Columns iFair, LFR, SSI, SSR, SF, and SF\_3 show the performances of the FCNN classifiers by applying iFair \cite{iFair}, LFR \cite{AIF360,lfr}, SenSeI \cite{SenSeI}, SenSR \cite{SensitiveSubspaceRobustness}, and our Siamese fairness approach on the baseline models, respectively. For the SF models, all the sensitive attribute values are used (whenever applicable) for augmentation, while for the SF\_3 models, only the maximum and minimum values of the sensitive attributes are used for augmentation to save computation consumption. The metrics with subscript $I(V)$ are computed over an augmented dataset $I(V)$, instead of an original (default) dataset $V$. For our SF and SF\_3 models, we use the Mean Squared Error loss function for the FCNN and LR classifiers, and the Hinge loss function for the SVM ones. The distance metrics in the accurate fairness constraints are all implemented with the Mean Absolute Error. An Adam optimizer \cite{adam} is deployed for training the FCNN and SVM classifiers, while a gradient descent optimizer \cite{SGD,SGD1} is for training the LR ones.

\figurename~\ref{pic:Compare_FCM_on_FCNN} and \figurename~\ref{pic:Fairea_FCNN} further demonstrates the fairness-accuracy trade-offs in terms of fairness confusion matrix performances and the Fairea evaluation \cite{fairea}, respectively.

\begin{table*}[!htb]
\begin{tabular}{c|ccccccc}
\hline
Metrics & BL & iFAIR & LFR & SSI & SSR & SF & SF\_3 \\ \hline
ACC & 0.871±0.007 & 0.763±0.047 & 0.814±0.020 & 0.851±0.017 & 0.648±0.076 & {0.874±0.004} & {0.876±0.004} \\ 
SPD & 0.105±0.015 & 0.110±0.095 & 0.096±0.059 & 0.058±0.010 & 0.077±0.085 & 0.111±0.015 & \multirow{6}{*}{——} \\ 
EOD & 0.046±0.011 & 0.109±0.092 & 0.124±0.054 & 0.038±0.010 & 0.086±0.089 & 0.051±0.016 &  \\ 
AVOD & 0.037±0.020 & 0.093±0.123 & 0.146±0.076 & 0.041±0.013 & 0.088±0.112 & 0.030±0.017 &  \\ 
SPD$_{I(V)}$ & 0.045±0.018 & 0.079±0.073 & 0.102±0.073 & 0.007±0.004 & 0.080±0.090 & {0.003±0.002} &  \\ 
EOD$_{I(V)}$ & 0.059±0.021 & 0.086±0.079 & 0.124±0.082 & 0.009±0.006 & 0.083±0.093 & {0.007±0.004} &  \\ 
AVOD$_{I(V)}$ & 0.081±0.028 & 0.067±0.097 & 0.149±0.101 & 0.011±0.006 & 0.087±0.111 & {0.005±0.003} &  \\ 
CON & 0.928±0.005 & 0.976±0.018 & 0.970±0.005 & 0.957±0.007 & 0.957±0.041 & 0.934±0.008 & 0.935±0.009 \\ 
FTA & 0.933±0.025 & 0.913±0.075 & 0.897±0.073 & 0.994±0.007 & 0.822±0.189 & 0.986±0.009 & 0.984±0.009 \\ 
TFR & 0.827±0.020 & 0.718±0.075 & 0.761±0.053 & 0.848±0.017 & 0.549±0.151 & {0.867±0.006} & {0.866±0.007} \\ 
TBR & 0.044±0.019 & 0.045±0.047 & 0.053±0.042 & 0.003±0.004 & 0.099±0.109 & 0.008±0.005 & 0.010±0.006 \\ 
FFR & 0.105±0.010 & 0.195±0.056 & 0.136±0.029 & 0.146±0.017 & 0.273±0.090 & 0.119±0.006 & 0.118±0.006 \\ 
FBR & 0.024±0.011 & 0.042±0.040 & 0.050±0.033 & 0.003±0.004 & 0.079±0.085 & 0.007±0.005 & 0.006±0.004 \\ 
F-R & 0.884±0.008 & 0.789±0.055 & 0.851±0.027 & 0.853±0.017 & 0.663±0.104 & 0.878±0.005 & 0.879±0.005 \\ 
F-P & 0.947±0.023 & 0.941±0.062 & 0.933±0.052 & 0.996±0.005 & 0.838±0.187 & 0.991±0.006 & 0.988±0.007 \\ 
F-F1 & 0.913±0.009 & 0.854±0.043 & 0.885±0.022 & 0.917±0.010 & 0.725±0.129 & {0.930±0.003} & {0.929±0.004} \\ \hline
\end{tabular}
\caption{Statistics of FCNN classifiers on the three fairness datasets}
\label{tab:average_result}
\end{table*}

\begin{figure}[!htb]
\centering
\includegraphics[scale=0.38, trim=0 30 0 30]{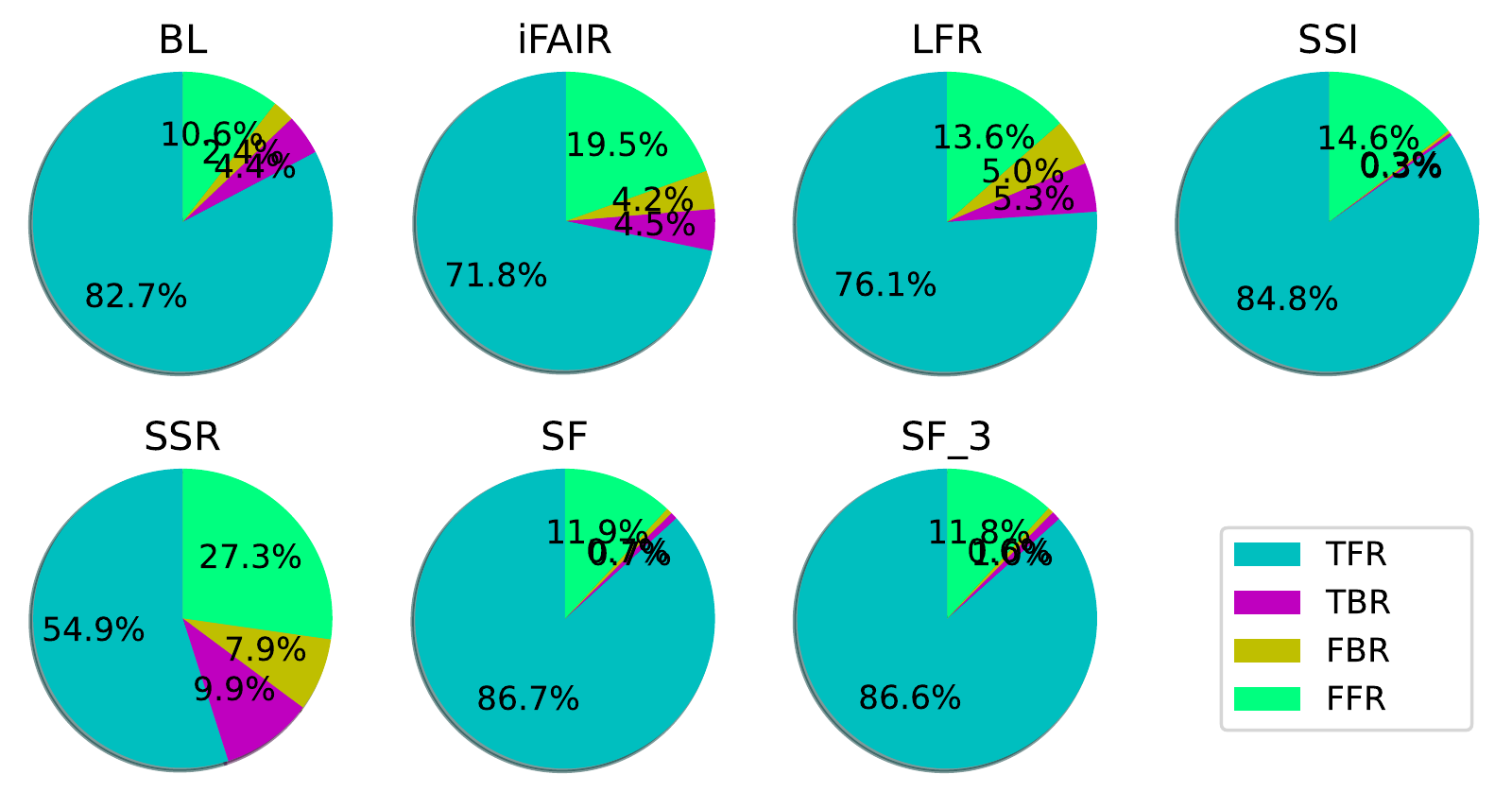}
\caption{Fairness Confusion Matrix performances (FCNNs)}
\label{pic:Compare_FCM_on_FCNN}
 \vskip -0.1in
\end{figure}
\begin{figure}[!htb]
\centering
\includegraphics[scale=0.20, trim=0 30 0 30]{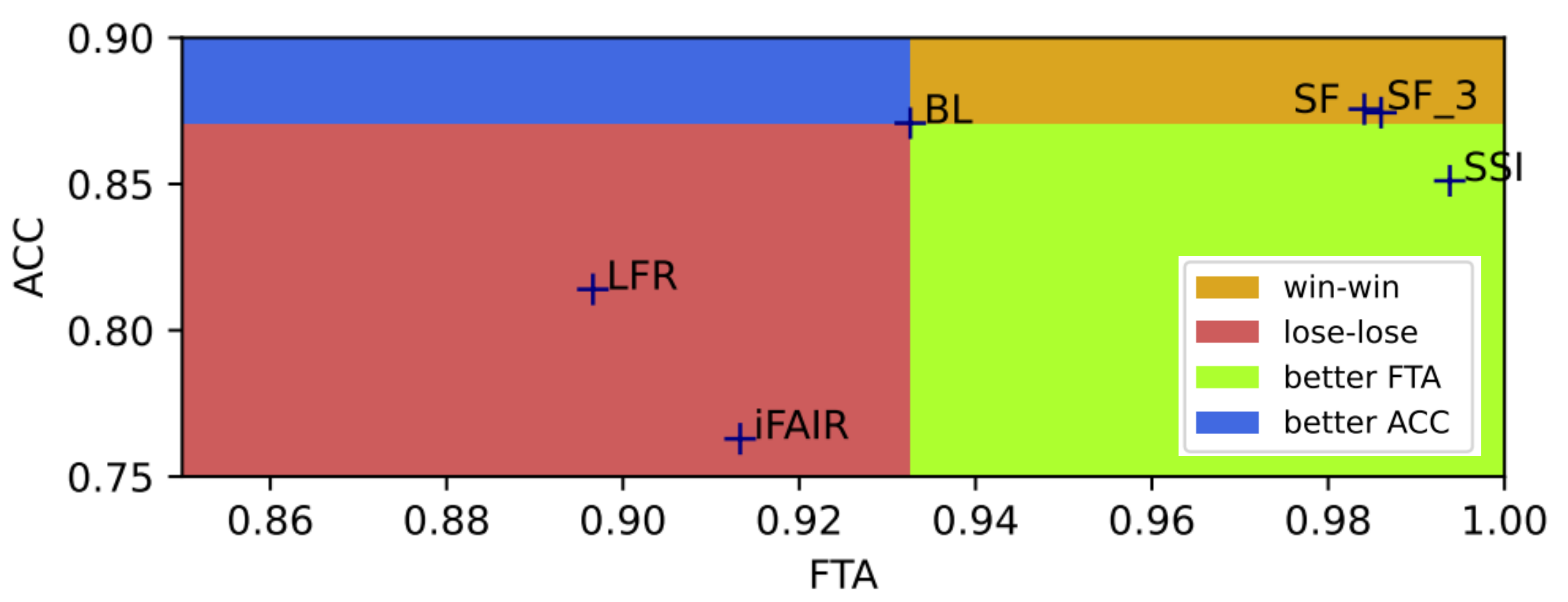}
\caption{Fairea evaluation (FCNNs)}
\label{pic:Fairea_FCNN}
\end{figure}

It can be seen in \tablename ~\ref{tab:average_result} and \figurename~\ref{pic:Compare_FCM_on_FCNN} that our Siamese fairness approach achieves the highest TFR and F-F1 performances, with both accuracy (ACC) and individual fairness (FTA) well improved. Compared with the state-of-the-art individual bias mitigation approaches, our Siamese fairness approach on average promotes 0.105 ACC (13.69\% higher ACC) and 0.080 FTA (8.78\% higher FTA) of a classifier. This is a direct consequence of the observation that our Siamese fairness approach promotes accurate fairness with on average 0.148 TFR (20.57\% higher TFR), reducing 0.043 TBR (85.03\% lower TBR), 0.037 FBR (85.07\% lower FBR), 0.069 FFR (36.53\% lower FFR), and promoting 0.085 F-F1 score (10.01\% higher F-F1 score).

The Fairea evaluation approach also certifies that only the SF and SF\_3 models fall into the win-win trade-off region, as shown in \figurename~\ref{pic:Fairea_FCNN}, which supplies a sufficiently strong signal on the bias mitigation effectiveness of our Siamese fairness approach. The other bias mitigation approaches may improve the individual fairness of a classifier but at the cost of its accuracy.
Moreover, our Siamese fairness approach also achieves the highest group fairness over the union of similar sub-populations. This suggests that accurate fairness can help reach a synergy between individual fairness and group fairness, such that improving group fairness can be manifested by improving individual fairness.


We herein discuss the performances of the FCNN classifiers. Similar observations can be made on the LR and SVM classifiers. Please refer to the supplementary material for their detailed experimental results.


\subsection{Service Discrimination with the Ctrip Dataset}
We then apply the accurate fairness criterion and the Siamese fairness approach to investigate a service discrimination problem, where customers with different consumption habits may be recommended disparate services, even though they pay the same prices for the same rooms. The Ctrip dataset includes 6 consumption habits attributes of customers (including the average time of order confirmation, the average advance days of booking, the average star level, class level, recommended level of hotels booked, and the average days of hotel stay) and 6 attributes of hotels (including order date, hotel ID, room type, room ID, star level and room price). For the service discrimination problem, the 6 customer attributes are designated as the sensitive attributes. The ground truth labels represent the room service types.

\begin{table}[!htb]
\resizebox{\linewidth}{!}{
\begin{tabular}{c|cccc}
\hline
Metric & BL(3) & SF\_3(3) & BL(5) & SF\_3(5) \\ \hline
ACC & 0.664±0.002 & 0.660±0.005 & 0.666±0.003 & 0.655±0.010 \\ 
CON & 0.940±0.009 & 0.978±0.008 & 0.927±0.010 & 0.988±0.013 \\ 
FTA & 0.524±0.107 & 0.902±0.022 & 0.352±0.188 & 0.958±0.041 \\ 
TFR & 0.412±0.068 & {0.620±0.013} & 0.281±0.145 & {0.637±0.010} \\ 
TBR & 0.251±0.069 & 0.040±0.013 & 0.385±0.146 & 0.018±0.018 \\ 
FFR & 0.112±0.039 & 0.282±0.010 & 0.071±0.043 & 0.321±0.032 \\ 
FBR & 0.225±0.038 & 0.058±0.012 & 0.263±0.042 & 0.024±0.023 \\ 
F-R & 0.792±0.030 & 0.688±0.004 & 0.804±0.031 & 0.666±0.020 \\ 
F-P & 0.621±0.103 & 0.940±0.019 & 0.423±0.218 & 0.972±0.027 \\ 
F-F1 & 0.689±0.052 & {0.794±0.006} & 0.516±0.206 & {0.790±0.006} \\ \hline
\end{tabular}
}
\caption{Statistics of FCNNs on the Ctrip datatset}
\label{tab:Statistics_ctrip} 
\end{table}

As reported in \tablename~\ref{tab:Statistics_ctrip}, the baseline (BL) models get an accuracy of 66.48\% on average, but only 34.68\% (TFR) customers are treated both accurately and fairly. Through Siamese fairness in-processing, the average TFR is extremely improved to 62.85\%, very close to its upper bound, which is the average accuracy of 65.77\%. Our Siamese fairness approach can make most (on average 93.00\%) of the customers to be fairly served irrespective of their consumption habits. Thus, for a further truthful promotion of their individual fairness, it is left to improve the accuracy of the classifiers themselves, instead of trading it.

\section{Conclusion}
We present in this paper the accurate fairness criterion, which is built upon the intuition that similar sub-populations shall be treated similarly up to the ground truths. Accurate fairness enhances individual fairness from the perspective of accuracy and paves the way to achieve harmony among accuracy, individual fairness, and group fairness. The accurate fairness criterion also induces a fairness confusion matrix that can identify the side effects of trading accuracy for individual fairness and vice versa, i.e., resulting in individually fair but faulty predictions (false fairness), or accurate but individually biased predictions (true bias).  

Then we present and evaluate the Siamese fairness in-processing approach in terms of fairness confusion matrix-based metrics. Our Siamese fairness approach aims to maximize the accurate fairness of a decision-making model with similar sub-populations as parallel training inputs. In this way, our Siamese fairness approach can significantly improve individual fairness without trading accuracy.

We envisage that the state-of-the-art bias mitigation techniques can be further refined from the perspective of accurate fairness. Studies on individual fairness usually rely on pre-specified sensitive attributes and disadvantaged groups. As part of future work, the fairness confusion matrix can be adapted to analyze which sensitive attributes pose more impacts on prediction outcomes. The accurate fairness criterion can be further utilized to help diagnose which groups under these attributes are treated unfavorably.

\bibliography{aaai23}
\clearpage

\section{Supplementary Material: Technical Appendix}
\subsection{Discussion on Theorem 3}
Group fairness concerns whether the groups with the same sensitive attributes values receive equal treatments. However, group fairness could be satisfied in a way such that individual fairness is unnecessarily neglected, e.g., by carefully selecting individuals who are unfavorably discriminated against in contrast to their similar counterparts.

By Theorem 3, if classifiers $f$ is accurately fair to each input in $W\subseteq V$, then 
\begin{eqnarray}
\mathbf{E}[D(f(\mathbf{X},a^*),f(\mathbf{X},a))] \leq  K\mathbf{E}[d((\mathbf{X},a^*),(\mathbf{X},a))] \label{A}\\
\mathbf{E}[D(\mathbf{Y},f(\mathbf{X},a))] \leq  K\mathbf{E}[d((\mathbf{X},a^*),(\mathbf{X},a))] \label{B}
\end{eqnarray}
over $I_a(W)$ with $(x,a^*)in W, a\in A$ and $a\neq a^*$. 

Inequality \eqref{A} suggests that under the accurate fairness criterion, the treatment differences between groups $I_{a^*}(W)$ and $I_{a}(W)$ are bounded by the distances between the individuals of the groups. This implies that accurate fairness can induce bounded disparity between the groups with different sensitive attributes values. 
Moreover, inequality \eqref{B} further concerns the ground-truths and suggests that under the accurate fairness criterion, 
each group enjoys a similar treatment close to the ground-truths, also bounded by the distances between their individuals. Thus, as an individual level fairness criterion, accurate fairness can help reach a synergy between individual fairness and group fairness.

\subsection{Performance Evaluation}
Herein, we present and analyze the performance evaluation results of the Logistics Regression (LR) and Support Vector Machine (SVM) classifiers on the \emph{Adult (Census Income)}, \emph{German Credit}, and \emph{ProPublica Recidivism (COMPAS)} datasets. 

\subsubsection{Logistics Regression}
We use the Mean Squared Error loss function and a gradient descent optimizer for training the LR classifiers, and set the Mean Absolute Error as the distance metrics in the accurate fairness constraints. 

\tablename~\ref{tab:statisticas_of_LR} reports the average statistics of ten runs for each bias mitigation approach compared over the three fairness datasets. Columns iFair, LFR, SF and SF\_3 show the performances of the LR classifiers resulted by applying iFair, LFR, and our Siamese fairness approach on the baseline models, respectively. \figurename~\ref{fig:statistics_of_LR}  and \figurename~\ref{fig:fairea_of_LR} further demonstrates the fairness-accuracy trade-offs in terms of fairness confusion matrix performances and the Fairea evaluation, respectively.

\begin{table*}[ht]
 \centering
 \setlength{\tabcolsep}{4mm}{
\begin{tabular}{c|c c c c c}
\hline
Metrics & BL & iFAIR & LFR & SF & SF\_3 \\ \hline
ACC & 0.846±0.005 & 0.765±0.052 & 0.819±0.010 & \textbf{0.863±0.012} & \textbf{0.855±0.010} \\ 
SPD & 0.090±0.007 & 0.106±0.114 & 0.045±0.034 & 0.081±0.021 &  \\ 
EOD & 0.044±0.021 & 0.115±0.118 & 0.019±0.040 & 0.031±0.019 &  \\ 
AVOD & 0.043±0.026 & 0.118±0.135 & 0.004±0.008 & 0.022±0.027 &  \\ 
SPD$_{I(V)}$ & 0.000±0.000 & 0.000±0.000 & 0.020±0.046 & \textbf{0.000±0.000} &  \\ 
EOD$_{I(V)}$ & 0.000±0.000 & 0.000±0.000 & 0.018±0.041 & \textbf{0.000±0.000} &  \\ 
AVOD$_{I(V)}$ & 0.000±0.000 & 0.000±0.000 & 0.004±0.007 & \textbf{0.000±0.000} & \multirow{-6}{*}{——} \\ 
CON & 0.991±0.004 & 0.967±0.030 & 0.998±0.002 & 0.959±0.021 & 0.972±0.011 \\ 
FTA & 1.000±0.000 & 1.000±0.000 & 0.980±0.046 & 1.000±0.000 & 1.000±0.000 \\ 
TFR & 0.846±0.005 & 0.765±0.052 & 0.804±0.046 & \textbf{0.863±0.012} & \textbf{0.855±0.010} \\ 
TBR & 0.000±0.000 & 0.000±0.000 & 0.015±0.037 & 0.000±0.000 & 0.000±0.000 \\ 
FFR & 0.154±0.005 & 0.235±0.052 & 0.176±0.003 & 0.137±0.012 & 0.145±0.010 \\ 
FBR & 0.000±0.000 & 0.000±0.000 & 0.005±0.010 & 0.000±0.000 & 0.000±0.000 \\ 
F-R & 0.846±0.005 & 0.765±0.052 & 0.823±0.004 & 0.863±0.012 & 0.855±0.010 \\ 
F-P & 1.000±0.000 & 1.000±0.000 & 0.983±0.042 & 1.000±0.000 & 1.000±0.000 \\ 
F-F1 & 0.914±0.003 & 0.864±0.034 & 0.889±0.026 & \textbf{0.924±0.007} & \textbf{0.919±0.006} \\ \hline
\end{tabular}
}
\caption{Statistics of LR classifiers on the three fairness datasets}
\label{tab:statisticas_of_LR}
\end{table*}
\begin{figure}[!htb]
\centering
\centerline{\includegraphics[scale=0.60, trim=0 10 0 10]{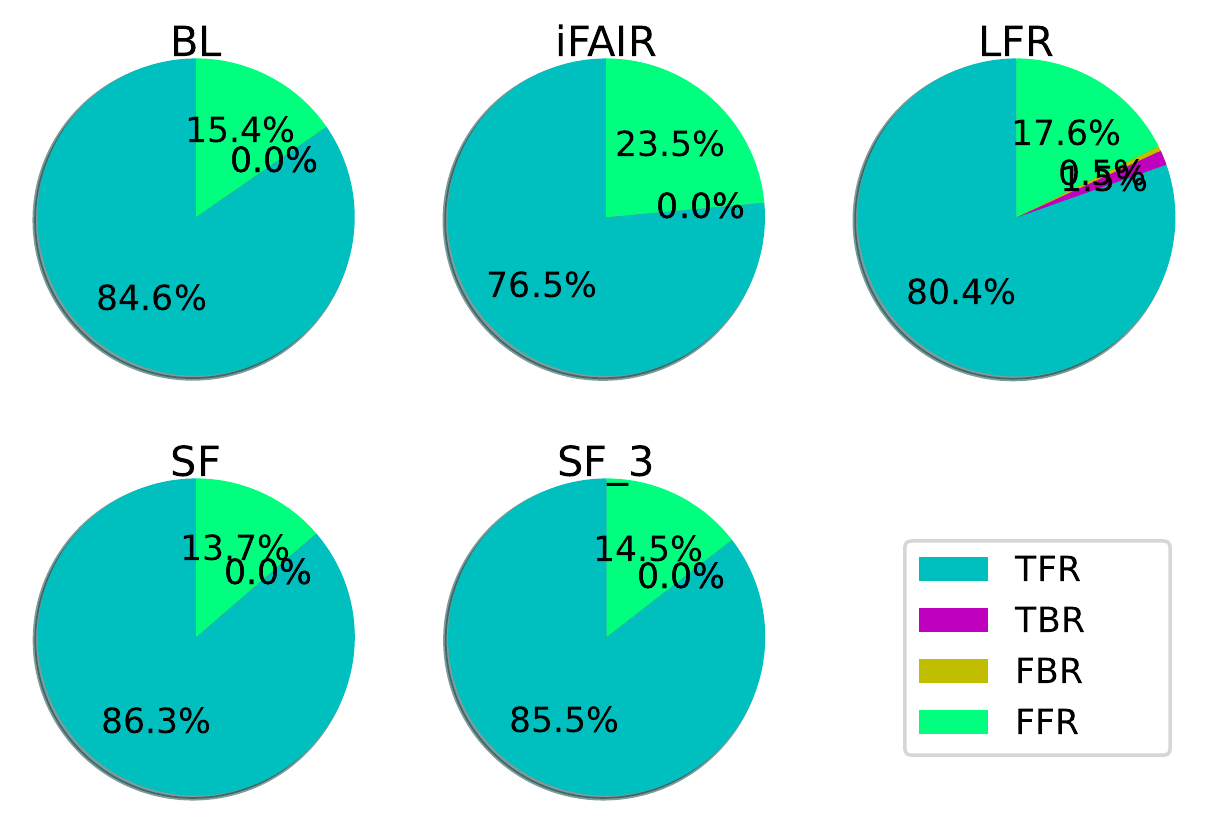}}
\caption{Fairness Confusion Matrix performances (LRs)}
\label{fig:statistics_of_LR}
\end{figure}
\begin{figure}[!htb]
 \centering
\centerline{\includegraphics[scale=0.60, trim=0 10 0 10]{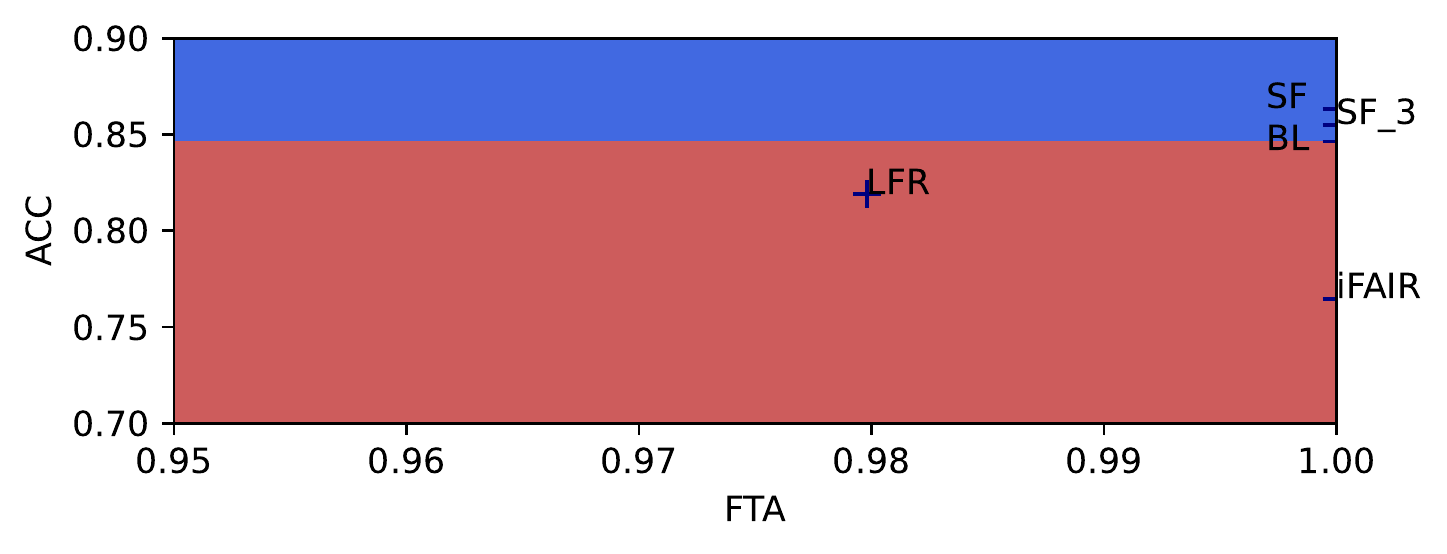}}
 \caption{Fairea evaluation (LRs)}
 \label{fig:fairea_of_LR}
\end{figure}

It can be seen in \tablename~\ref{tab:statisticas_of_LR} and \figurename~\ref{fig:statistics_of_LR} that our Siamese fairness approach achieves the highest TFR and fair-F1 score performances on the LR classifiers with both accuracy (ACC) and individual fairness (FTA) well improved. Compared with other individual bias mitigation techniques, our Siamese fairness approach on average promotes 0.072 ACC (9.03\%  higher ACC) and 0.01 FTA (1.02\% higher FTA), which is a direct consequence of the observation that our Siamese fairness approach promotes on average 0.079 TFR (10.09\% higher TFR), reducing 0.008 TBR (100\% lower TBR), 0.003 FBR (100\% lower FBR), 0.068 FFR (33.33\% lower FFR), and promoting 0.048 F-F1 score (5.43\% higher F-F1 score).

In \figurename~\ref{fig:fairea_of_LR}, the Fairea evaluation approach locates the SF and SF\_3 models on the right borderline of the better ACC (blue) region, which is exactly the win-win trade-off region with FTA=1.00. This shows that only our Siamese fairness approach achieves higher ACC without sacrificing FTA. Moreover, our Siamese fairness approach also achieves the highest full group fairness over the union of the similar sub-populations, indicating that the improvement of accurate fairness leads not only to the improvement of individual fairness, but also to the improvement of group fairness.

\subsubsection{Support Vector Machines}
For the SVM classifiers, we use the Hinge loss function and an Adam optimizer, and similarly, set the Mean Absolute Error as the distance metrics in the accurate fairness constraints. 

The statistical evaluation results, fairness confusion matrix performances and Fairea evaluation are demonstrated in \tablename~\ref{tab:statisticas_of_SVM}, \figurename~\ref{fig:statistics_of_SVM} and \figurename~\ref{fig:fairea_of_SVM}, respectively.
\begin{table*}[!htb]
 \centering
 \setlength{\tabcolsep}{4mm}{
\begin{tabular}{c|c c c c c}
\hline
Metric & BL & iFAIR & LFR & SF & SF\_3 \\ \hline
ACC & \textbf{0.827±0.014} & 0.722±0.076 & 0.792±0.026 & 0.821±0.015 & 0.817±0.013 \\ 
SPD & 0.054±0.031 & 0.091±0.077 & 0.040±0.038 & 0.045±0.031 &  \\ 
EOD & 0.025±0.032 & 0.094±0.073 & 0.037±0.039 & 0.015±0.029 &  \\ 
AVOD & 0.035±0.051 & 0.124±0.089 & 0.028±0.031 & 0.018±0.044 &  \\ 
SPD$_{I(V)}$ & 0.016±0.027 & 0.084±0.077 & 0.044±0.044 & \textbf{0.006±0.021} &  \\ 
EOD$_{I(V)}$ & 0.021±0.034 & 0.086±0.073 & 0.044±0.042 & \textbf{0.008±0.024} &  \\ 
AVOD$_{I(V)}$ & 0.035±0.055 & 0.115±0.092 & 0.027±0.030 & \textbf{0.013±0.038} & \multirow{-6}{*}{——} \\ 
CON & 0.992±0.011 & 0.982±0.018 & 0.989±0.007 & 0.996±0.006 & 0.996±0.005 \\ 
FTA & 0.939±0.108 & 0.914±0.076 & 0.954±0.044 & 0.979±0.036 & 0.983±0.017 \\ 
TFR & 0.779±0.085 & 0.670±0.081 & 0.763±0.049 & \textbf{0.807±0.034} & \textbf{0.806±0.023} \\ 
TBR & 0.048±0.087 & 0.052±0.048 & 0.030±0.032 & 0.014±0.027 & 0.012±0.014 \\ 
FFR & 0.160±0.028 & 0.244±0.077 & 0.192±0.020 & 0.172±0.012 & 0.177±0.009 \\ 
FBR & 0.013±0.024 & 0.034±0.034 & 0.016±0.014 & 0.007±0.012 & 0.005±0.006 \\ 
F-R & 0.836±0.024 & 0.737±0.089 & 0.803±0.024 & 0.827±0.011 & 0.822±0.010 \\ 
F-P & 0.946±0.102 & 0.926±0.076 & 0.965±0.039 & 0.985±0.030 & 0.987±0.015 \\ 
F-F1 & 0.874±0.068 & 0.812±0.062 & 0.871±0.027 & \textbf{0.894±0.019} & \textbf{0.893±0.012} \\ \hline
\end{tabular}
}
\caption{Statistics of SVM classifiers on the three fairness datasets}
\label{tab:statisticas_of_SVM}
\end{table*}

\begin{figure}[!htb]
 \centering
\centerline{\includegraphics[scale=0.60, trim=0 10 0 10]{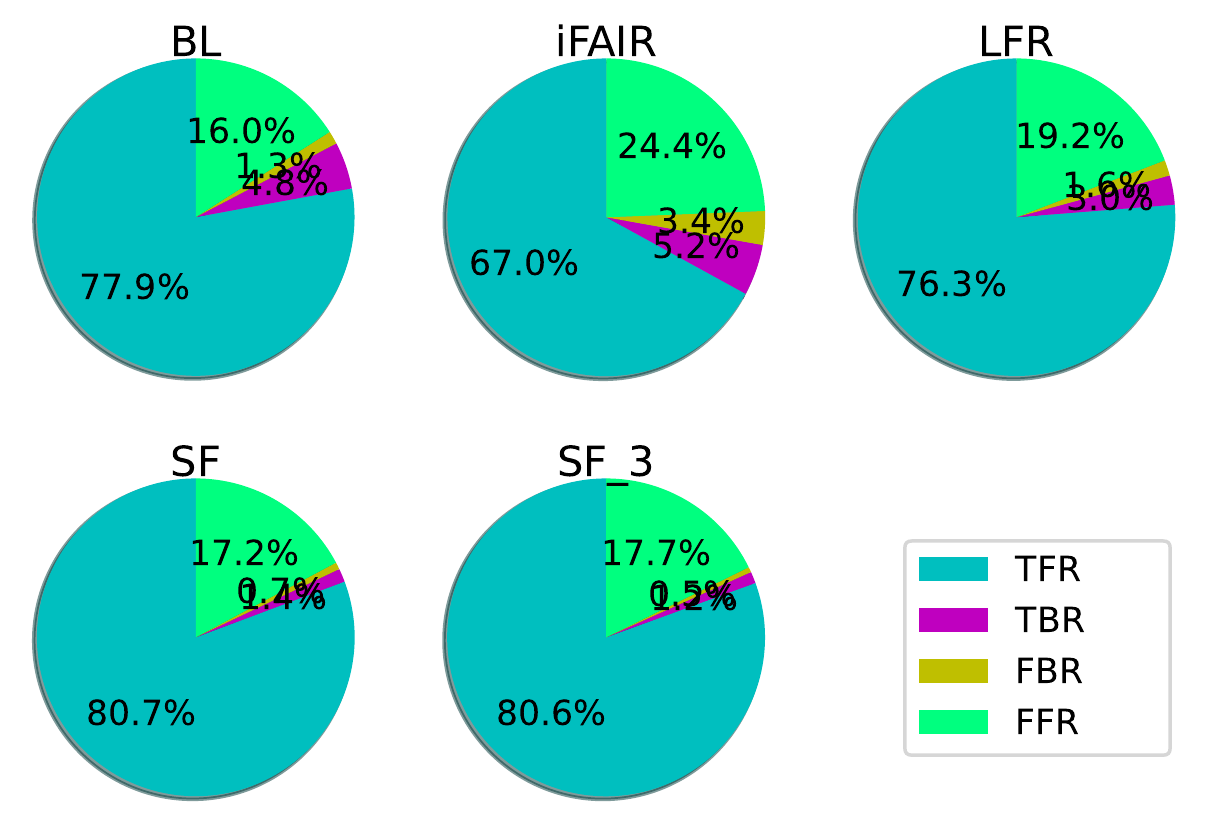}}
 \caption{Fairness Confusion Matrix performances (SVMs)}
 \label{fig:statistics_of_SVM}
\end{figure}
\begin{figure}[!htb]
 \centering
\centerline{\includegraphics[scale=0.60, trim=0 10 0 10]{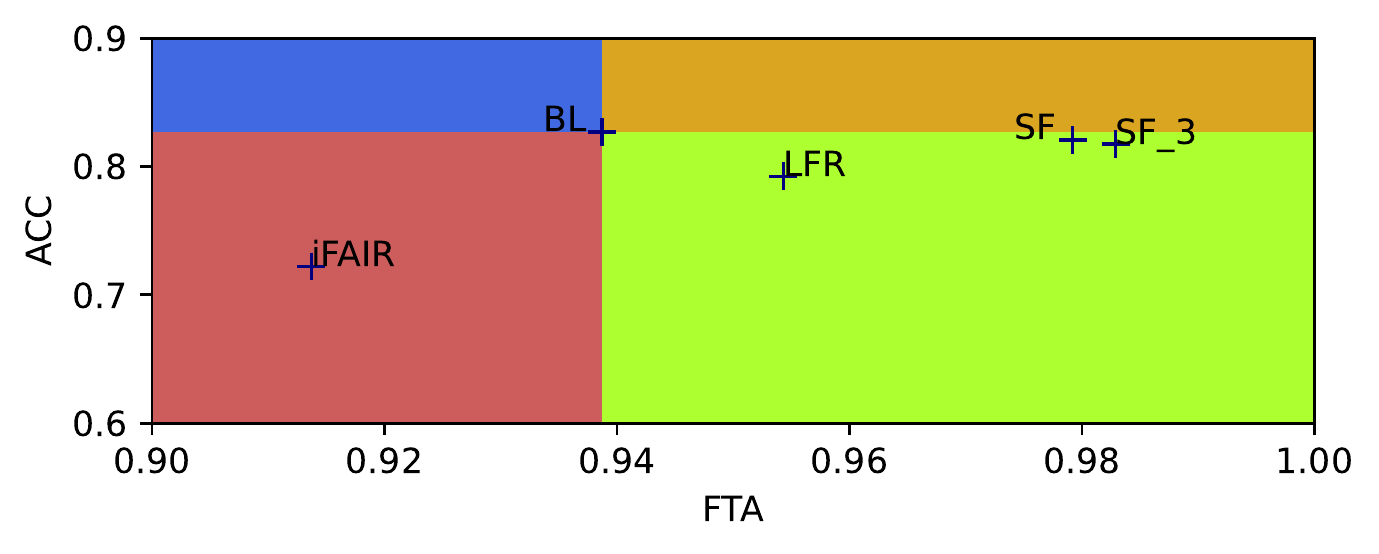}}
 \caption{Fairea evaluation (SVMs)}
 \label{fig:fairea_of_SVM}
\end{figure}
Our Siamese fairness approach also achieves the highest TFR and fair-F1 score performances on the SVM classifiers, with individual fairness (FTA) improved and accuracy (ACC) maintained, as shown in  \tablename~\ref{tab:statisticas_of_SVM} and \figurename~\ref{fig:statistics_of_SVM}. Compared with the state-of-the-art individual bias mitigation techniques, our Siamese fairness approach on average promotes 0.063 ACC (8.38\% higher ACC), 0.045 FTA (4.84\% higher FTA) and 0.052 F-F1 score (6.22\% higher F-F1 score), as a result of promoting 0.091 TFR (12.67\% higher TFR), reducing 0.027 TBR (66.75\% lower TBR), 0.018 FBR (71.31\% lower FBR), and 0.046 FFR (21.10\% lower FFR). 

\figurename~\ref{fig:fairea_of_SVM} shows that the Fairea evaluation approach locates the SF and SF\_3 models in the better FTA (green) region near the border-line of the win-win (yellow) region, as SF and SF\_3 achieve better individual fairness (FTA) with a slight accuracy loss on the SVM classifiers. Similarly, our Siamese fairness approach also achieves the highest group fairness over the union of the similar sub-populations by improving individual fairness.

\end{document}